\RequirePackage{silence}
\documentclass[letterpaper, 10 pt, journal, twoside]{IEEEtran}

\usepackage{silence,graphicx,float,caption,dsfont,amsfonts,color}
\usepackage{amsmath,amsfonts, amssymb,commath,algorithm,algpseudocode,comment,booktabs}%
\usepackage[noadjust]{cite}

\usepackage{enumitem}

\usepackage{footnote}
\makesavenoteenv{algorithm}

\makeatletter
\let\NAT@parse\undefined
\makeatother
\usepackage[hypertexnames=false,hidelinks]{hyperref}
\usepackage{cleveref}
\usepackage[all]{hypcap}

\renewcommand{\vec}{\mathbf}
\newcommand{\trs}{\top}
\newcommand{\RR}{\mathbb R}

\newcommand{\A}{\mathcal A}

\newcommand{\F}{\mathcal F}
\newcommand{\G}{\mathcal G}

\newcommand{\N}{\mathcal N}
\renewcommand{\O}{\mathcal O}
\newcommand{\R}{\mathcal R}
\renewcommand{\S}{\mathcal S}
\newcommand{\V}{\mathcal V}
\newcommand{\X}{\mathcal X}
\newcommand{\Y}{\mathcal Y}

\graphicspath{{Figures/}}
\crefname{equation}{}{}

\DeclareMathOperator*{\argmax}{arg\,max}

\let\oldforall\forall
\let\forall\undefined
\DeclareMathOperator{\forall}{\oldforall}

\newtheorem{lemma}{Lemma}
\newtheorem{theorem}{Theorem}
\newtheorem{problem}{Problem}
\newtheorem{remark}{Remark}[section]
\newtheorem{definition}{Definition}
\newtheorem{assumption}{Assumption}
\newtheorem{proof}{Proof}

\renewcommand{\secref}[1]{Section~\ref{#1}}
\renewcommand{\figref}[1]{Fig.~$\ref{#1}$}
\renewcommand{\algref}[1]{Algorithm~$\ref{#1}$}
\newcommand{\linref}[1]{line~$\ref{#1}$}
\newcommand{\linsref}[2]{lines~$\ref{#1}$-$\ref{#2}$}
\renewenvironment{proof}{\begin{IEEEproof}}{\end{IEEEproof}\ignorespacesafterend}
\renewcommand{\IEEEproofindentspace}{1\parindent}
\usepackage{subfigure}

\long\def\ignorethis#1{}

\title{Distributed Resilient Submodular Action Selection in Adversarial Environments}

\author{Jun Liu$^{1}$, Lifeng Zhou$^{2}$, Pratap Tokekar$^{3}$, and Ryan K. Williams$^{1}$
\thanks{This work was supported by the National Institute of Food and Agriculture under Grant 2018-67007-28380, the Office of Naval Research under Grant N00014-18-1-2829, and the National Science Foundation under Grant 1943368.}
\thanks{$^{1}$The authors are with the Department of Electrical and Computer Engineering, Virginia Polytechnic Institute and State University, Blacksburg, VA 24061 USA (e-mail: junliu@vt.edu; rywilli1@vt.edu).}
\thanks{$^{2}$The author was with the Department of Electrical
and Computer Engineering, Virginia Polytechnic Institute and State University, Blacksburg, VA 24061 USA when part of the work was completed. He is currently with the GRASP Laboratory,
University of Pennsylvania, Philadelphia, PA 19104 USA (e-mail: lfzhou@seas.upenn.edu).}
\thanks{$^{3}$The author is with the Department of Computer Science, University of Maryland, College Park, MD 20782 USA (e-mail: tokekar@umd.edu).}
\thanks{Digital Object Identifier (DOI): see top of this page.}}

\begin{document}

\bstctlcite{IEEEexample:BSTcontrol}
\maketitle
\thispagestyle{empty}
\pagestyle{empty}

\begin{abstract}
    In this letter, we consider a distributed submodular maximization problem for multi-robot systems when attacked by adversaries. One of the major challenges for multi-robot systems is to increase resilience against failures or attacks. This is particularly important for distributed systems under attack as there is no central point of command that can detect, mitigate, and recover from attacks. Instead, a distributed multi-robot system must coordinate effectively to overcome adversarial attacks. In this work, our distributed submodular action selection problem models a broad set of scenarios where each robot in a multi-robot system has multiple action selections that may fulfill a global objective, such as exploration or target tracking. To increase resilience in this context, we propose a \emph{fully} distributed algorithm to guide each robot's action selection when the system is attacked. The proposed algorithm guarantees performance in a worst-case scenario where up to a portion of the robots malfunction due to attacks. Importantly, the proposed algorithm is also consistent, as it is shown to converge to the same solution as a centralized method. Finally, a distributed resilient multi-robot exploration problem is presented to confirm the performance of the proposed algorithm. 
\end{abstract}


\section{Introduction}
\label{sec:intro}

Resilience is a crucial property for multi-robot systems.
Consider, for example, a multi-robot exploration application where each robot selects exploration actions from an action candidate set, e.g., a motion primitive set. In adversarial environments, sensors may fail or get attacked, and depending on the contributions of the attacked sensors, the exploration performance may be seriously affected. This problem is more challenging in distributed multi-robot systems since each robot can only share its local information with neighbors to maximize the system reward subject to adversarial influences.

This letter focuses on a scenario where the robots in a distributed multi-robot system need to work together to guard the system against worst-case attacks. By worst-case attacks, we refer to the case where the system may have up to $K$ sensor failures.
Robots operating in adversarial scenarios may get cyber-attacked or face failures, resulting in a temporary withdrawal of robots from the task (e.g., because of temporary deactivation of their sensors, blockage of their field of view). For example, in \figref{fig: application}, each robot in the system is equipped with a downward-facing camera to explore an environment with different weights in different areas, where there is one robot whose \emph{sensor} is blocked by an attacker. It is worth mentioning that robot failure and sensor failure are different. If a sensor is attacked, the corresponding robot may not know this attack and still perform other tasks/communications as planned.

\begin{figure}[!tbp]
    \centering
    \includegraphics[width=2.9in]{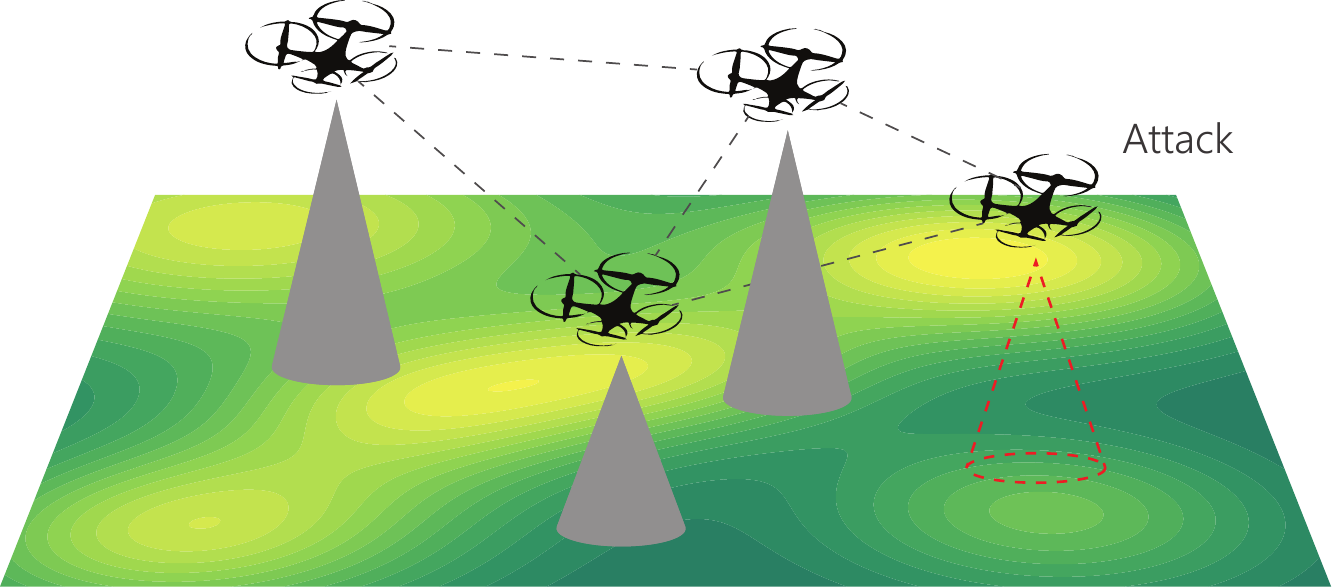}
    \caption{In a multi-robot environmental exploration application, the robots are mounted with downward-facing cameras to explore the environment.  An attacker blocks one of the robots' cameras (red).}
    \label{fig: application}
\end{figure}

\emph{Related Work:} The resilience of multi-robot systems has received attention recently (see a comprehensive survey in \cite{zhou2021multi}). In \cite{saulnier2017resilient}, the authors presented a resilient formation control algorithm that steers a team of mobile robots to achieve desired flocking even though some team members are non-cooperative (or adversarial) and broadcast deceptive signals. Deceptive or spoofing attacks were also considered in the wireless communication~\cite{gil2017guaranteeing} and the target state estimation~\cite{mitra2019resilient} of multi-robot teams. Another type of attack, called masquerade attack, was studied in a multi-agent path-finding problem~\cite{wardega2019masquerade}. In this letter, we instead focus on defending multi-robot systems against the denial-of-service (DoS) attacks that can compromise the sensors' functionality~\cite{raymond2008denial}. For example, a polynomial-time resilient algorithm to counter adversarial denial-of-service (DoS) attacks or failures in a submodular maximization problem was proposed in \cite{tzoumas2017resilient}. Meanwhile, resilient coordination algorithms have been designed to cope with adversarial attacks in multi-robot target tracking \cite{zhou2018resilient}, the orienteering problem \cite{shi2020robust}, etc. In \cite{zhou2020distributed}, the authors proposed to solve the centralized resilient target tracking problem \cite{zhou2018resilient} in a distributed way. This method partitions robots into subgroups/cliques, and then the subgroups perform a centralized algorithm in parallel to counter the worst-case attacks. Thus, even if there exist communications between subgroups, these available communications are not utilized because each subgroup operates independently. Therefore, the proposed algorithm in \cite{zhou2020distributed} has a worse approximation bound than its centralized counterpart.

The action selection problem falls into the combinatorial robotics application domain.
The authors in \cite{choi2009consensus} proposed a consensus-based method for the task allocation problem. In \cite{qu2019distributed}, the authors used matroids to model the task allocation constraints and provided a distributed approach with $1/2$ optimality ratio. In \cite{williams2017decentralized}, the authors extended the use of matroids to abstract task allocation constraints modeling and demonstrated the suboptimality through a sequential auction method in a decentralized scenario. In \cite{liu2020monitoring,liu2020coupled}, the authors applied submodular and matroids techniques in a multi-robot intermittent environmental monitoring problem, where the deployment actions are selected based on the environmental process. In \cite{corah2019distributed}, the authors utilized the submodularity of a mutual information function to prove the performance of a distributed multi-robot exploration method, while synchronization is needed. The authors in \cite{liu2019submodular} considered two coupled action selection problems in an environmental monitoring application, where the selected tasks have an impact on the monitored environmental process behavior. In \cite{grimsman2018impact}, the authors studied how the information from other robots impact the decisions of a multi-robot system in distributed settings. Similarly, the submodular properties were also utilized in the consensus problem \cite{mackin2018submodular}, the leader selection problem \cite{clark2013supermodular}, etc.
However, resilience is not the primary consideration, especially when the system is under worst-case attacks. In this letter, we propose a fully distributed resilient algorithm that requires no central point of command to solve the action selection problem in adversarial environments. The proposed distributed resilient method can perform as well as the corresponding centralized algorithm when subject to worst-case adversarial attacks.

\emph{Contributions}: The contributions are as follows:
\begin{enumerate}
    \item We formulate a fully distributed resilient submodular action selection problem.
    \item We demonstrate how to solve the problem in a \emph{fully} distributed manner where each robot computes its action only and shares the decision with its neighbors to achieve convergence with performance guarantees.
    \item We prove and evaluate the proposed algorithm's performance is consistent, as it is shown to converge to the same solution as a centralized method.
\end{enumerate}

\emph{Organization:}
In \secref{sec:problem}, we introduce preliminaries followed by the problem formulation. In \secref{sec:algorithm}, we use two subsections to demonstrate the two phases of the proposed algorithm. Then, the performance analysis of the proposed algorithm is shown in \secref{sec:performance}. In \secref{sec:simulation}, numerical evaluation is performed in a multi-robot exploration problem. We then close the letter in \secref{sec: conclusion}.

\section{Preliminaries and Problem Formulation}
\label{sec:problem}

\subsection{Submodular Set Functions}
\label{ssec: submodular}

\begin{definition}[\textit{Submodularity} \cite{nemhauser1978analysis}]
    A submodular function $f: 2^\V \mapsto \RR$ is a set function, satisfying the property $f(\X \cup \{v\}) - f(\X) \ge f(\Y \cup \{v\}) - f(\Y)$,
    where $\V$ is the ground set, $\X \subseteq \Y \subseteq \V$, and $v \in \V \setminus \Y$.
\end{definition}

The power set $2^\V$ is the set of all subsets of $\V$, including $\emptyset$ and $\V$ itself. A set function is monotone non-decreasing if $f(\X) \le f(\Y)$ when $\X \subseteq \Y \subseteq \V$. Submodularity appears in a wide variety of robotics applications. We refer the reader to \cite{krause2014submodular,schrijver2003combinatorial} for more details.

\begin{definition}[\textit{Marginal gain}]
    For a set function $f: 2^\V \mapsto \RR$, let the marginal gain of adding element $v \in \V$ into set $\X \subseteq \V$ be $f_\X(v) \triangleq f(\X \cup \{v\}) - f(\X)$\footnote{We will use $v$ to represent $\{v\}$ if there is no confusion.}.
\end{definition}

\begin{definition}[\textit{Curvature} \cite{conforti1984submodular}]
    Let $f: 2^\V \mapsto \RR$ be a monotone non-decreasing submodular function, we define the curvature of $f(\cdot)$ as $c_f = 1- \min_{v \in \V} \frac{f(\V) - f(\V \setminus \{v\})}{f(v)}$, where $v \in \V$.
\end{definition}

This curvature represents the submodularity level of $f(\cdot)$. It holds that $0 \le c_f \le 1$. If $c_f = 0$, then $f(\cdot)$ is a \emph{modular} function and $f(\X \cup \{v\}) - f(\X) = f(v)$. If $c_f = 1$, then $f(\X \cup \{v\}) - f(\X) = 0$, where $\X \subseteq \V$ and $v \in \V \setminus \X$. 


\subsection{Problem Formulation}

\begin{figure}[!tbp]
    \centering
    \includegraphics[width=2.9in]{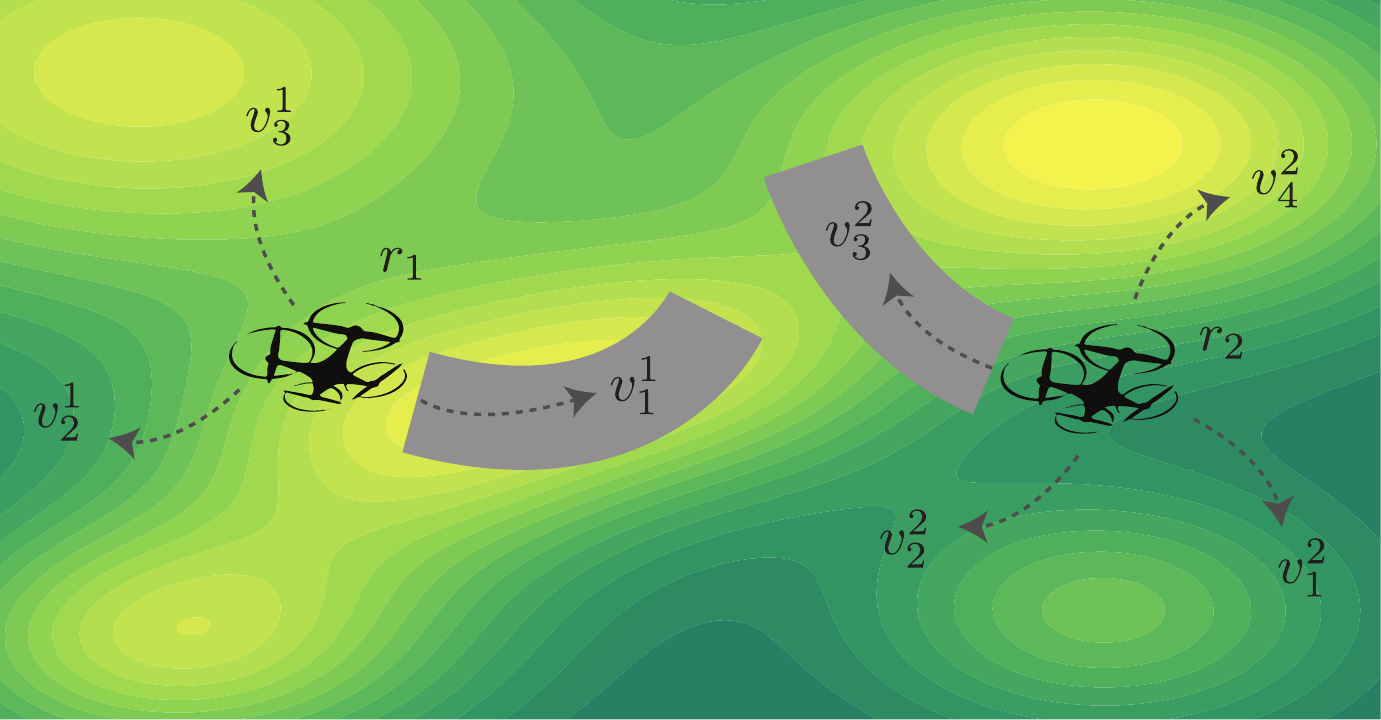}
    \caption{Each robot chooses one motion primitive from its candidate motion primitive set to explore a region of the environment.}
    \label{fig: example}
\end{figure}

\textit{Robots, actions, and rewards:} Consider a team of $N$ robots denoted by $\R = \{1, \ldots, N\}$. Each robot is equipped with one sensor.
There is a (undirected) communication graph\footnote{It is worth noting that the proposed algorithm also works for directed communication graphs.} $\G=(\mathcal{R},\mathcal{E})$ associated with nodes $\mathcal{R}$, and edges $\mathcal{E}$ such that $(i,j)\in\mathcal{E}$ if $i$ and $j$ can communicate with each other. We denote by $\N_i$ the neighbors of robot $i$. The diameter $d(\G)$ of the communication $\G$ is the greatest length of the shortest paths between vertices. The communication can be synchronized or asynchronized. The performance analysis will be based on synchronized communication.
Each robot $i \in \R$ has a set of candidate actions $\V_i$ and can only choose one action from $\V_i$ at each execution step.
For example, in motion planning using motion primitives, the robot can only choose one motion primitive from its candidate motion primitives at a time. As shown in Fig.~\ref{fig: example}, robot 1 chooses action $v_1^1$ from its available action set $\V_1 = \{v_1^1,v_2^1,v_3^1\}$ and robot 2 chooses action $v_3^2$ from its available action set $\V_2 = \{v_1^2, v_2^2, v_3^2, v_4^2\}$, yielding the shaded explored area. We denote by $\mathcal{V} \triangleq \bigcup_{i\in \mathcal{R}} \mathcal{V}_i$ the ground set containing all robots' possible actions. There is an reward associated with any action. Also, the reward associated with action $v_1^1$ is the gray explored area. The function value $f(\S)$ or the combined reward associated with $v_1^1$ and $v_3^2$ is the explored gray areas.

\textit{Objective Function:} We use a non-decreasing and submodular function $f: 2^{\mathcal{V}} \to \mathbb{R}$ to model the quality of each valid action set $\mathcal{S} \subseteq \mathcal{V}$ since the diminishing return property of objective functions is common in robotics. For example, in Fig.~\ref{fig: example}, $f(\cdot)$ measures the extent of the joint area explored by chosen actions $\mathcal{S}=\{v_1^1, v_3^2\}$ (represented by the gray areas), which is a well-known coverage function that exhibits the submodularity property \cite{krause2014submodular}.

\begin{assumption}[\textit{Attacks}]
    We assume the robots encounter worst-case attacks that result in their sensor DoS failures. Thus, robots can still communicate with their neighbors even though their sensors are denied or blocked. The maximum number of anticipated attacks is upper bounded by $K, (K\leq N)$, where $N$ is the number of robots.
\end{assumption}

\begin{problem}[\textit{Distributed resilient multi-robot action selection in adversarial environments}]\label{pro:dis_resi_sub}
The robots, by communicating actions and rewards with their neighbors over the communication graph $\G$, choose action set $\mathcal{S}$ (per robot per action) to maximize a submodular objective $f(\cdot)$ against $K$ worst-case attacks. That is
\begin{align} \label{eq:dis_resi_sub}
    \begin{split}
        \underset{\S \subseteq \V}{\text{maximize}} \quad &\min_{\F \subseteq \S} f(\mathcal{S}\setminus\mathcal{F})\\
        \text{subject to}\quad &|\mathcal{S}\cap \mathcal{V}_i|= 1, \forall i\in \mathcal{R}, \\
        & |\mathcal{F}| \leq K,\\
    \end{split}
\end{align}
where $\R$ contains the indexes of the robots in the system, $\mathcal{F}$ denotes the action set associated with the attacked sensors, and $\V_i$ is the available action set for robot $i$.
\end{problem}

The first constraint ensures that robot $i$ only chooses one action from its action set $\mathcal{V}_i$. The ``$\min$'' operator indicates the attacks we consider are the worst-case attacks.
The constraint $|\mathcal{F}|\leq K$ captures the problem assumption that at most $K$ sensors in the team can fail or get attacked.

In this problem, each robot needs to take other robots' actions into consideration while making its decision. That is because neighboring robots' selected actions may have an impact on local robot's action selection. In other words, different action selections result in different performances.

\section{A Consistent Algorithm for Distributed Resilient Submodular Maximization}
\label{sec:algorithm}

\begin{algorithm}[!t]
    \caption{Distributed resilient selection for robot $i$}
    \label{alg: 1}
    
    \textbf{Input: }
    \begin{itemize}
        \item Action set $\V_i$; number of anticipated attacks $K$; 
        \item Communication graph $\G$; objective function $f(\cdot)$.
    \end{itemize}
    
    \textbf{Output:} Set $\S$.
    
    \begin{algorithmic}[1]
        \State $\S^i_1 \leftarrow \emptyset, \S^i_2 \leftarrow \emptyset, \alpha_1^i \leftarrow 0, \alpha_2^i \leftarrow 0$;
        \State $\S^i_1 \leftarrow $\textsc{GenerateRemovals$(\S^i_1, \alpha_1^i)$};
        \State $\S^i_2 \leftarrow $\textsc{GenerateComplements$(\S^i_1, \S^i_2, \alpha_2^i)$};
        \State $\S \leftarrow \S^i_1 \cup \S^i_2$.
    \end{algorithmic}
\end{algorithm}

We present a distributed resilient algorithm (\algref{alg: 1}) for solving Problem 1. 
At a high level, \algref{alg: 1} contains two main procedures \textsc{GenerateRemovals} (\algref{alg: 2}) and \textsc{GenerateComplements} (\algref{alg: 3}). In the following, we present and analyze these procedures from robot $i$'s perspective since other robots will follow the same procedures. In general, robot $i$ will use these two procedures to approximate the following two sets:
\begin{itemize}
    \item $\S^i_1$: the set that \emph{approximates} the optimal worst-case removal set. Since computing the optimal worst-case removal set is intractable, we use $\S^i_1$ as an approximation. We denote by $\A$ the indices of the robots used by $\S^i_1$. This is the phase I.
    \item $\S^i_2$: the set that \emph{approximates} the optimal set that maximizes the objective function using $\V \setminus \V_i, \forall i \in \A$. Again, this is an approximation since computing the optimal set is intractable. This is phase II.
\end{itemize}
These two procedures will be executed sequentially. Robot $i$ will use $\alpha^i_1$ and $\alpha^i_2$ as two counters for different phases to decide whether to stop the corresponding procedure or not.
Upon the stopping of \algref{alg: 1}, both $\S^i_1$ and $\S^i_2$ converge.
The final solution of Problem 1 will then be $\S^i_1 \cup \S^i_2$. 

In each phase, robot $i$ will approximate and update $\S^i_1$ and $\S^i_2$ through the following processes:
\begin{enumerate}
    \item \emph{Initialization}, which is used to make the first approximation.
    \item \emph{Inter-robot communication}, which is used to combine its local approximation with neighbors' approximations.
    \item \emph{Local computation}, which is used to update local approximation.
\end{enumerate}

\subsection{Phase I: Generate Approximated Removals}
\label{ssec: generate removals}

\begin{algorithm}[!t]
    \caption{(Phase I) Generate approximated removal set for each robot $i$}
    \label{alg: 2}
    \begin{algorithmic}[1]
        \Procedure{GenerateRemovals}{$\S^i_1, \alpha_1^i$}
        \While{$\alpha^i_1 < 2d(\G)$}
        \If{$\S^i_1 = \emptyset$} \Comment{\emph{1) Initialization}} \label{lin: 21}
        \State $\S^i_1 \leftarrow \argmax_{v \in \V_i} f(v)$; 
        \State $f(s) \leftarrow \max_{v \in \V_i} f(v)$;
        \EndIf

        \State
        \State $\S^i_1 \leftarrow \S^i_1 \cup \S^j_1, \forall j \in \N_i$; \Comment{\emph{2) Communication}}
        \State $M = \min(K, |\S^i_1|)$; \Comment{\emph{3) Local computation}} \label{lin: min(k,s)}
        \State $\S^i_1 \leftarrow$ top $M$ actions \footnote{Top $M$ actions in action set $\S^i_1 (|\S^i_1| \geq M)$: given the function values $f(s)$ of all actions $s \in \S^i_1$, sort these function values in a descending order, and set the $M$ actions corresponding to the first $M$ function values as the top $M$ actions in $\S^i_1$.} in $\S^i_1$; \label{lin: select top M}

        \State send $(\S^i_1, \{f(s)\}), \forall s \in \S^i_1$ to all $j \in \N_i$;
        \State update $\alpha^i_1$.
        \EndWhile
        \EndProcedure
    \end{algorithmic}
\end{algorithm}

The procedure for generating approximated removals is called \textsc{GenerateRemovals} (\algref{alg: 2}).
This procedure aims to approximate $K$ action removals $\S_1$ $(|\S_1|=K)$ through the below processes.


\emph{1) Initialization:}
In \algref{alg: 2}, robot $i$ first selects an action that contributes the most to the objective function $f$ regardless of other robots' selections. The selected action is $s \in \argmax_{v \in \V_i} f(v)$. Following the constraint $|\S \cap \V_i| = 1, \forall i \in \R$, robot $i$ is only allowed to select one action from its candidate action set $\V_i$ to update its action set $\S^i_1$. Meanwhile, $f(s)$ is also recorded.


\emph{2) Inter-robot communication:}
To update $i$'s local approximation set $\S^i_1$, robot $i$ needs to combine $j$'s approximation $\S^j_1, \forall j \in \N_i$. Since our task in this phase is to approximate $K$ action removals, we can merge $j$'s approximation as $\S^i_1 \leftarrow \S^i_1 \cup \S^j_1$.


\emph{3) Local computation:}
Once receiving neighbor $j$'s action set $\S^j_1$, robot $i$ updates its action set $\S^i_1$ based on $\S^j_1$. We first form a new candidate set $\S^i_1 \leftarrow \S^i_1 \cup \S^j_1$ to update $\S^i_1$. 
Then, we need to select the top $K$ actions for robot $i$. There are two cases:
\begin{itemize}
    \item If $|\S^i_1| \leq K$, there is no need to update $\S^i_1$.
    \item If $|\S^i_1| > K$, robot $i$ selects the top $K$ actions from $\S^i_1$.
\end{itemize}
In \algref{alg: 2} \linsref{lin: min(k,s)}{lin: select top M}, we combine these two cases as a single operation. That is, robot $i$ needs to select top $m := \min(K, |\S^i_1|)$ actions from $\S^i_1$.
Finally, robot $i$ shares $\S^i_1$ and the corresponding action values $f(s), \forall s \in \S^i_1$ with all its neighbors $j \in \N_i$.


\emph{4) Stopping condition:} After one cycle of local computation and inter-robot communication, the local counter $\alpha^i_1$ will be incremented by $1$. Finally, when $\alpha^i_1$ reaches $2d(\G)$, robot $i$ stops and all robots have an agreement on $\S_1$.

\subsection{Phase II: Generate Approximated Complements}
\label{ssec: generate complements}

\begin{algorithm}[!t]
    \caption{(Phase II) Generate complements for robot $i$}
    \label{alg: 3}

    \begin{algorithmic}[1]
        \Procedure{GenerateComplements}{$\S^i_1, \S^i_2, \alpha_2^i$}
        \While{$\alpha^i_2 < 2d(\G)$}
        \If{$\S^i_2 = \emptyset$} \Comment{\emph{1) initialization}}
        \State $s \in \argmax_{v \in \V_i} f_\emptyset(v)$;
        \State $\S^i_2 \leftarrow \{s\}$;
        \EndIf

        \State
        \For{$j \in \N_i$} \Comment{\emph{2) inter communication}}
        \State $\S_2^{i+} \leftarrow \mathsf{sort}(\{\S_2^i, \S_2^j\}, \text{`descend'})$; \label{lin: sort}
        \State $\S_2^{i+} \leftarrow$ remove redundant actions in $\S_2^{i+}$; \label{lin: remove redundant}
        \State $\S_2^{i+} \leftarrow$ remove order changed actions in $\S_2^{i+}$;\label{lin: remove order changed}
        \State $\S_2^i \leftarrow \S_2^{i+}$;
        \EndFor \label{lin: 32}
        
        \State
        \State $\X \leftarrow \emptyset$; \Comment{\emph{3) local computation}} \label{lin: local first}
        \For{$s_n \in \S_2^i, n=1, \ldots, |\S^i_2|$}
        \State $g \leftarrow f_{\{s_1, \ldots, s_{n-1}\}}(s_n)$;
        \If{$f_\X(\X \cup v) \ge g, \forall v \in \V_i$}
        \State $s \in \argmax_{v \in \V_i} f_\X(\X \cup v)$;
        \State $\X \leftarrow \X \cup \{s\}$;
        \State \textbf{break};
        \EndIf
        \State $\X \leftarrow \X \cup \{s_n\}$.
        \EndFor \label{lin: 36}
        \State $\S^i_2 \leftarrow \X$; \label{lin: local last}
        
        \State
        \State send $\S^i_2$ and marginal gains of $s \in \S^i_2$ to $\N_i$;
        \State update $\alpha^i_2$.
        \EndWhile
        \EndProcedure
    \end{algorithmic}
\end{algorithm}

The procedure for generating the approximated complements is \textsc{GenerateComplements} shown in \algref{alg: 3}. This procedure aims to approximate $N-K$ greedy action selections $\S_2$ ($|\S_2| = |\R \setminus \A| = N-K$) for the remaining robots $\mathcal{R}\setminus \mathcal{A}$ through inter-robot communication and local computation with $\A$ denoting the robots that select actions in phase I. Depending on whether robot $i$ is used as removals or not, robot $i$ in phase II will have two different functionalities:
\begin{itemize}
    \item If $i \in \A$, then robot $i$ acts as a conveyor only to merge the approximation $\S^j_2$ from $j, \forall j \in \N_i$ and broadcast the merged/updated $\S^i_2$ to $j, \forall j \in \N_i$. So, robot $i$ only participated in \emph{inter-robot communication}.
    \item If $i \in \R \setminus \A$, robot $i$ also needs to update its approximation $\S^i_2$ using the \emph{local computation} process.
\end{itemize}

In the following, we demonstrate phase II from robot $i$'s perspective assuming $i \in \R \setminus \A$. If $i \in \A$, then the local computation process will be skipped for robot $i$.


\emph{1) Initialization:}
At the first iteration of phase II, $\S^i_2 = \emptyset$ and thus robot $i$ can directly update $\S^i_2$ as the action with the maximum marginal gain based on the empty set. That is, $s \in \argmax_{v \in \V_i} f_{\emptyset}(v)$.
Then, $\S^i_2$ is updated as $\S^i_2 \leftarrow s$.
The corresponding marginal gain $f_{\emptyset}(s)$ is also recorded.
%


\emph{2) Inter-robot communication:}
Let us consider the case where $|\S^i_2| = n$ with $n \leq N - K$ at some point before the algorithm stops.
We first consider $s \in \S^i_2$ in the descending order they are added through the local computation procedure.
For example, if $|\S^i_2| = n$, we can write $\S^i_2$ as $\S^i_2 = \{ s_1, \ldots, s_n\}$,
such that $f_\emptyset(s_1) \ge \ldots \ge f_{\{s_1, \ldots, s_{n-1}\}} (s_n)$.
We also use $\gamma(\cdot)$ to denote the order of action $s \in \S^i_2$ as
\begin{equation*}
    \gamma(s_1) = 1, \ldots, \gamma(s_n) = n.
\end{equation*}

Similarly, we also apply this reordering procedure to $\S_2^j, \forall j \in \N_i$. Thus, there is also a marginal gain and an order associated with the action $s \in \S^j_2$. With the marginal gains and orders ready, we are ready to merge $\S_2^j$ with $\S^i_2$.
For every $\S^j_2$, we augment $\S^i_2$ with $\S_2^j$ and apply an operation as   
\begin{equation*}
    \S_2^{i+} \leftarrow \mathsf{sort}(\{\S_2^i, \S_2^j\}, \text{`descend'}).
\end{equation*}
This operation is read as ``$s \in \{\S^i_2, \S_2^j\}$ are sorted in a descending order based on the associated marginal gains''.

\textbf{Remove redundant actions:} 
The merged set $\S^{i+}_2$ may contain \emph{redundant} actions.
By redundant actions, we refer to the actions $s \in \S^{i+}_2$ having the following \emph{redundant action} properties:
\begin{itemize}
    \item $s$ appears in $\S^i_2$ and $\S_2^j$. e.g., $s = s'$ where $s \in \S^i_2$ and $s' \in \S_2^j$; 
    \item The associated marginal gains are the same.
    \item The orders in $\S^i_2$ and $\S_2^j$ are the same. e.g., $\gamma(s) = \gamma(s')$.
\end{itemize}
We can check these properties for each $v \in \S^{i+}_2$ against all other actions to remove the redundant actions.


\textbf{Remove order changed actions:}
After the above process, we know that there is an order associated with $s, \forall s \in \S^{i+}_2$. Similarly, we also know that $s \in \S^i_2$ and $s \in \S^j_2$ also have their orders. If the order of any action is changed before and after the augmentation, this action and the actions having lower marginal gains than this action's marginal gain will be invalid.
This rule is from the submodularity of $f(\cdot)$. We can then use the following properties to remove order changed actions. Specifically, we need to remove any $s \in \S^{i+}_2$ if $s$ satisfies the following \emph{order changed} properties:
\begin{itemize}
    \item $s$ appears in $\S^{i+}_2$ and $\S_2^i$ (or $\S_2^j$). e.g., $s = s'$ where $s \in \S^{i+}_2$ and $s' \in \S_2^i$ (or $s' \in \S_2^j$).
    \item The orders are not the same. e.g., $\gamma(s) \neq \gamma(s')$.
\end{itemize}

This operation can be illustrated by the following example. If the local approximation $\S^i_2$ is
\begin{equation*}
    \S^i_2 = \{s_1, s_2\} \quad \text{and} \quad f_\emptyset (s_1) \ge f_{s_1} (s_2),
\end{equation*}
Then, the orders of $s_1, s_2 \in \S^i_2$ are as follows $\gamma(s_1) = 1, \gamma(s_2) = 2$.
Also, if a neighbor $j$'s approximation is
\begin{equation*}
    \S_2^j = \{s_2, s_3\} \quad \text{and} \quad f_\emptyset (s_2) \ge f_{s_2} (s_3).
\end{equation*}
Then, the orders of $s_2, s_3 \in \S_2^j$ are $\gamma(s_2) = 1, \gamma(s_3) = 2$.
Now consider the case where the augmented set is $\S^{i+}_2 = \{s_1, s_2, s_3\}$,
and the marginal gains are such that
\begin{equation}
    f_\emptyset (s_1) \ge \underbrace{f_{s_1} (s_2)}_{s_2 \in \S^i_2} = \underbrace{f_\emptyset (s_2)}_{s_2 \in \S^j_2} \ge f_{s_2} (s_3).
    \label{eq: order}
\end{equation}
After applying the above redundancy removal procedure, there may exist a case where
\begin{equation*}
    f_{\{s_1, s_2\}}(v) > f_{\{s_1, s_2\}}(s_3), 
\end{equation*}
where $v \in \V \setminus (\S^i_1 \cup \{s_1, s_2\})$. In this case, $s_3$ is no longer a valid action in $\S^{i+}_2$ as action $v$ has a higher marginal gain. 
To deal with this case, we can use action orders. From the marginal gains relations in \eqref{eq: order}, we have
\begin{equation*}
    \gamma(s_1) = 1, \underbrace{\gamma(s_2)}_{\text{$s_2 \in \S^i_2$}} = 2, \underbrace{\gamma(s_2)}_{\text{$s_2 \in \S^j_2$}} = 3, \gamma(s_3) = 4.
\end{equation*}
Also, we know that the original orders of $s_2, s_3 \in \S^j_2$ are
\begin{equation*}
    \gamma(s_2) = 1, \gamma(v_3) = 2, \quad \text{where} \quad s_2, s_3 \in \S^j_2.
\end{equation*}
So, the order of $s_2$ and $s_3$ where $s_2, s_3 \in \S^j_2$ are changed after merging. Therefore, we need to remove these two actions from $\S^{i+}_2$. Finally, the augmented approximation is assigned to $\S^i_2$ as $\S^i_2 \leftarrow \S^{i+}_2$.


\emph{3) Local computation:}
After the inter-robot communication process, robot $i$ may need to change its original action selection. That is because robot $i$ made its selection before knowing its neighbors' approximations $(\S^j_2, \forall j \in \N_i)$. Once receiving $\S^j_2, \forall j \in \N_i$, robot $i$ can update its own action selection, i.e., $v \in \V_i$.

We update robot $i$'s action selection based on the marginal gain of $v \in \V_i$ for every possible combination of its neighbors' selections. The necessity of this operation is from the observation that the marginal gain of an action will be changed if the already selected action set is changed.
For example, after the inter-robot communication, if we have $\S^i_2 = \{s_1, \ldots, s_n\}$ and $v \notin \S^i_2, \forall v \in V_i$, 
we then need to check the marginal gain of $v \in \V_i$ when $v$ has different orders in $\S^i_2$. Since we already know the associated marginal gains of $s \in \S^i_2$, we can compare
\begin{equation*}
\begin{split}
    \max_{v \in \V_i} f_{\emptyset}(v) \quad \text{v} & \text{s.} \quad  f_{\emptyset}(s_1),\\
    \quad & \vdots \\
    \max_{v \in \V_i} f_{\{s_1, \ldots, s_{n-1}\}}(v) \quad  \text{v} & \text{s.} \quad f_{\{s_1, \ldots, s_{n-1}\}}(s_n).
    \end{split}
\end{equation*}
Whenever $\forall v \in \V_i$ generates a better marginal gain than the compared one, we replace the compared action with the action $v$ and delete the actions selected after the compared one. This is because if the orders of the actions are changed, then the marginal gains are invalid. This operation is shown in \linsref{lin: local first}{lin: local last} (\algref{alg: 3}). 
Meanwhile, the associated marginal gain is also updated. Finally, the updated $\S^i_2$ along with the marginal gains of $s \in \S^i_2$ are broadcasted to $j \in \N_i$.


\emph{4) Stopping condition:}
After one cycle of local computation and inter-robot communication, the local counter $\alpha^i_2$ will be incremented by $1$ if there is no change of $\S^i_2$ before and after these two processes. Otherwise, the counter $\alpha^i_2$ is reset to $0$. When $\alpha^i_2$ reaches $2d(\G)$, where $d(\G)$ is the diameter of $\G$, robot $i$ stops operations. Meanwhile, all robots have an agreement on the approximation of $\S_2$.

\section{Performance Analysis}
\label{sec:performance}

\begin{lemma}
    The procedure \textsc{GenerateRemovals} (\algref{alg: 2}) for finding the approximated removals has the following performance:
    \begin{enumerate}
        \item \emph{Approximation performance:} The approximated removals for robot $i$ is $\S^i_1 = \S_1$, where $\S_1$ is the $K$-max consensus result.
        \item \emph{Convergence time:} The algorithm takes $d(\G)$ steps to converge, where $d(\G)$ is the diameter of $\G$.
        \item \emph{Computational complexity:} The computational complexity for every robot is at most $\O(|\V_i|)$.
    \end{enumerate}
    \label{lem: 1}
\end{lemma}

\begin{proof}
    1). \emph{Approximation performance:}
    In the centralized scenario, we know that we need to find the top $K$ actions to approximate the removal set. In the distributed scenario, $\S^i_1$ is updated as $\S^i_1 \leftarrow \argmax_{v \in \V_1(i)} f(v)$ at the beginning.
    Assume that $i$ and $j$ are different before communicating with each other. Upon receiving $\S^j_1$, robot $i$'s approximation $\S^i_1$ is updated by using $\min(K, |\S^i_1 \cup \S^j_1|)$ actions as shown in \linref{lin: min(k,s)} (\algref{alg: 2}).
    Similarly, this procedure is also applied to $j$.
    Thus, robot $i$ and $j$ will agree with each other on the top $K$ actions after communication.
    Finally, when all robots $r \in \R$ receive other robots' approximation after $d(\G)$ steps, they achieve a consensus on the top $K$ actions.

    2) \emph{Convergence time:}
    In every execution of \algref{alg: 2}, $i$ needs to update $\S^i_1$ through the received $\S^j_1$ from $j \in \N_i$. Similarly, it takes $d(\G)$ steps for $i$ to receive $\S^r_1$ from $r$ that has the longest communication distance. During this procedure, every robot can receive all other robots' approximation at least once.
    Thus, \algref{alg: 2} takes $d(\G)$ steps to converge.

    3) \emph{Computational complexity:}
    Robot $i$ needs $|\V_i|$ evaluations to find the largest contribution $s \in \argmax_{v \in \V_i} f(v)$. After communications, if $s \in \S^i_1$, then $s$ is among the top $K$ actions. If $s \notin \S^i_1$, then $s$ is replaced by other actions with larger contributions and we do not need to evaluate for robot $i$ again. So, the number of evaluations for every robot is at most $\O(|\V_i|)$.
\end{proof}

\begin{lemma}
    The procedure \textsc{GenerateComplements} (\algref{alg: 3}) for finding the complements has the following performance:
    \begin{enumerate}
        \item \emph{Approximation performance:} The approximated complements for $i$ is $\S^i_2 = \S_2$, where $\S_2$ is the centralized greedy solution generated by using marginal gains.
        \item \emph{Convergence time:} In the worst-case, the algorithm converges in $2(N-K+1)d(\G)$ steps, where $d(\G)$ is the diameter of $\G$.
        \item \emph{Computational complexity:} The computational complexity for every robot is at most $\O((N-K)^2 |\V_i|)$.
    \end{enumerate}
    \label{lem: 2}
\end{lemma}

\begin{proof}
    1). \emph{Approximation performance:}
    When merging $\S_2^j$ with $\S^i_2$, we first use the $\mathsf{sort}(\cdot)$ procedure to maintain the orders of the actions $s \in \S^{i+}_2$ regardless of the redundancy and the orders of the actions in $\S^{i+}_2$ as shown in \linref{lin: sort} of \algref{alg: 3}.
    Then, through the operation described in \secref{ssec: generate complements} (also in \linsref{lin: remove redundant}{lin: remove order changed} of \algref{alg: 3}), we resolve these two issues by removing any $s \in \S^{i+}_2$ that is either redundant or order changed. In local computation, when robot $i$ updates its action set, the marginal gains of $s \in \S^i_2 \setminus \V_i$ are used as oracles. In the local computation procedure, when any $v \in \V_i$ replaces the compared action, the actions having lower marginal gains in $\S^i_2$ are removed from $\S^i_2$.
    This procedure maintains the descending orders of $s \in \S^i_2$ while updating robot $i$'s contribution. Therefore, all these procedures help to keep the descending order of $s \in \S^i_2$.
    When these procedures are applied to all robots in $\R$, the system converges to the same approximation $\S^i_2$ since every robot will have an agreement on at least one action after each communication. Also, since the descending orders of $v \in \S^i_2$ are kept during all communications, the final converged $\S^i_2$ is the same as the centralized solution. That is, $\S^i_2 = \S_2$.

    2) \emph{Convergence time:}
    Through the above analysis, we know that if $\S^i_2 \neq \S_2^j$, then this disagreement is resolved through communications. In an extreme case, we assume that the communication distance between robot $i$ and robot $r$ is $d(\G)$.
    It then takes $2d(\G)$ steps for $i$ to agree with $r$ on at least one action that is selected by $i$ or $r$. Also, $\max_{i \in \R} |\S^i_2| = N-K$. Therefore, it takes at most $2(N-K) d(\G)$ steps to reach the final agreement.
    Meanwhile, robot $i$ needs to take another $2d(\G)$ steps to confirm the convergence.

    3) \emph{Computational complexity:}
    \algref{alg: 3} needs at most $|\V_i| |\S^i_2 \cup \S_2^j|$ evaluations during each local computation procedure since $i$ checks its maximum contribution against every combination of the actions in the merged set $\S^i_2 \cup \S_2^j$. Also, it holds that $\max_{i,j \in \R} |\S^i_2 \cup \S_2^j| = N-K$.
    Therefore, the computational complexity for every robot is at most $\O((N-K)^2 |\V_i|)$.
\end{proof}

\begin{theorem}
    \algref{alg: 1} has the following performance:
    \begin{enumerate}
        \item \emph{Performance:} The approximation ratio is
              \begin{equation*}
                  f(\S \setminus \F^\star) \ge \max \{\frac{1-c_f}{1 + c_f}, \frac{1}{1+K}, \frac{1}{|\mathcal{R}|-K}\} f(\S^\star \setminus \F^\star).
              \end{equation*}
              where $\S^\star$ is an optimal solution and $\F^\star$ is an optimal removal set with respect to $\S^\star$.
        \item \emph{Convergence time:} In the worst-case, the algorithm converges in $(2N-2K+3)d(\G)$ steps, where $d(\G)$ is the diameter of $\G$.
        \item \emph{Computational complexity:} The computational complexity for every robot is at most $\O((N-K)^2 |\V_i|)$.
    \end{enumerate}
\end{theorem}

\begin{proof}
    1). \emph{Approximation performance:}
    From \lemref{lem: 1} and \lemref{lem: 2}, we know that $\S^i_1 = \S_1$ and $\S^i_2 = \S_2$, where $\S_1$ and $\S_2$ are the corresponding centralized solutions. Then, the approximation performance of the distributed resilient algorithm (Algorithm~\ref{alg: 1}) is the same as that of its centralized counterpart \cite{tzoumas2017resilient}. That is,
    \begin{equation*}
                  f(\S \setminus \F^\star) \ge \max \left\{\frac{1-c_f}{1 + c_f}, \frac{1}{1+K}, \frac{1}{|\mathcal{R}|-K} \right\} f(\S^\star \setminus \F^\star).
    \end{equation*}
    where $\S^\star$ is an optimal solution and $\F^\star$ is an optimal removal set with respect to $\S^\star$.

    2) \emph{Convergence time:}
    Based on the results from \lemref{lem: 1} and \lemref{lem: 2}, we know that the convergence time is $(2N-2K+3)d(\G)$.

    3) \emph{Computational complexity:}
    Combining the results from \lemref{lem: 1} and \lemref{lem: 2}, we have the computational complexity as $\O((N-K)^2 |\V_i|)$.
\end{proof}

\section{Numerical Evaluation}
\label{sec:simulation}

\subsection{Simulation Setup}

\emph{Environment settings:} We verify the performance of \algref{alg: 1} by implementing it into a scenario where a distributed multi-robot system explores an environment modeled by a Gaussian mixture model (GMM).
Specifically, the environment is generated as $z(x,y) = \sum_{\ell = 1}^B r_\ell b_\ell(x,y) = \vec r^\trs \vec b$, where $(x,y)$ are 2D coordinates, $r_\ell: \RR \mapsto \RR$ are the weights for the basis functions $b_\ell: \RR^2 \mapsto \RR, \forall \ell = 1, \ldots, B$. Also, $\vec r = [r_1, \ldots, r_\ell]^\trs$ and $\vec b = [b_1, \ldots, b_\ell]^\trs$ are the stacked weights and basis functions respectively. The number of basis function and variances are selected randomly. In the simulation, we use a $200 \times 200$ field to represent the environment. There is an environmental importance associated with each location. The importance value of a location equals to the GMM value of that location.

\emph{Compared algorithms:}
We compare the performance of \algref{alg: 1}, which is referred to as ``\emph{distributed-resilient}", with the performance of the following methods:
\begin{itemize}
    \item An \emph{optimal} method, where the solution is generated through a brute-force search.
    \item A \emph{semi-dist} method \cite{zhou2020distributed}, where the solution is generated by first partitioning robots into groups and then running a centralized resilient algorithm in each group.
    \item A \emph{cent-greedy} method \cite{nemhauser1978analysis}, where the solution is generated greedily based on marginal gains that maximize the objective function in a centralized manner.
    \item A \emph{cent-rand} method, where the solution is generated randomly in a centralized manner.
\end{itemize}

\emph{Multi-robot system settings:} 
We compare the performance of the system using two different settings:
1), In the first setting, we compare the performance of our distributed resilient method with the optimal, the semi-dist, the cent-greedy method, and the cent-rand methods. We set the number of robots as $N =5$, and the maximum number of attacks as $K = 3$. Then, we run $200$ trials to compare the performance. We generate random initial locations for the robots in each trial, and each robot has four actions (forward, backward, left, and right). 2), In the second setting, we compare the performance of the resilient method, the semi-dist, the cent-greedy method, and the cent-rand method. Specifically, we set the number of robots as $N \in \{30, 40, 50\}$. The corresponding number of attacks $K$ is randomly generated from $[0.5N, 0.75N]$ and rounded to an integer. This setting means that at least $50\%$ of the robots will be attacked, and at most $75\%$ of the robots will be attacked. The robots' rewards are added with white Gaussian noise with a mean of $10\%$ of the original rewards and a variance of $5\%$ of the original rewards. We then run $50$ trials for each setting. However, since finding the worst-cast attacks is intractable and we aim to test the proposed algorithm's scalability, we use greedy attacks in this case. That is, the attacker attacks the robots' sensors greedily using the standard greedy algorithm~\cite{nemhauser1978analysis}. Common simulation parameters include: in each of the above settings, the sensing range of the robots is set to $10$; the reward of an action is the environmental importance explored by this action. In each trial, the robots are randomly placed in a region with $x \in [50, 100], y \in [50, 100]$. Finally, we perform Monte Carlo simulations to test the performance of these four methods with the same simulation parameters.

\subsection{Performance Comparison}

\emph{Performance metric:}
The performance of different methods is captured by the sum of the importance in the explored area after worst-case attacks. Specifically, we first generate a solution for each method. Then, attackers attack and remove the contributions of attacked robots. Since finding worst-case attacks is intractable, we use a brute-force search to find the worst-case attacks for each generated solution.

\begin{figure}[!tbp]
    \centering
    \includegraphics[width=2.9in]{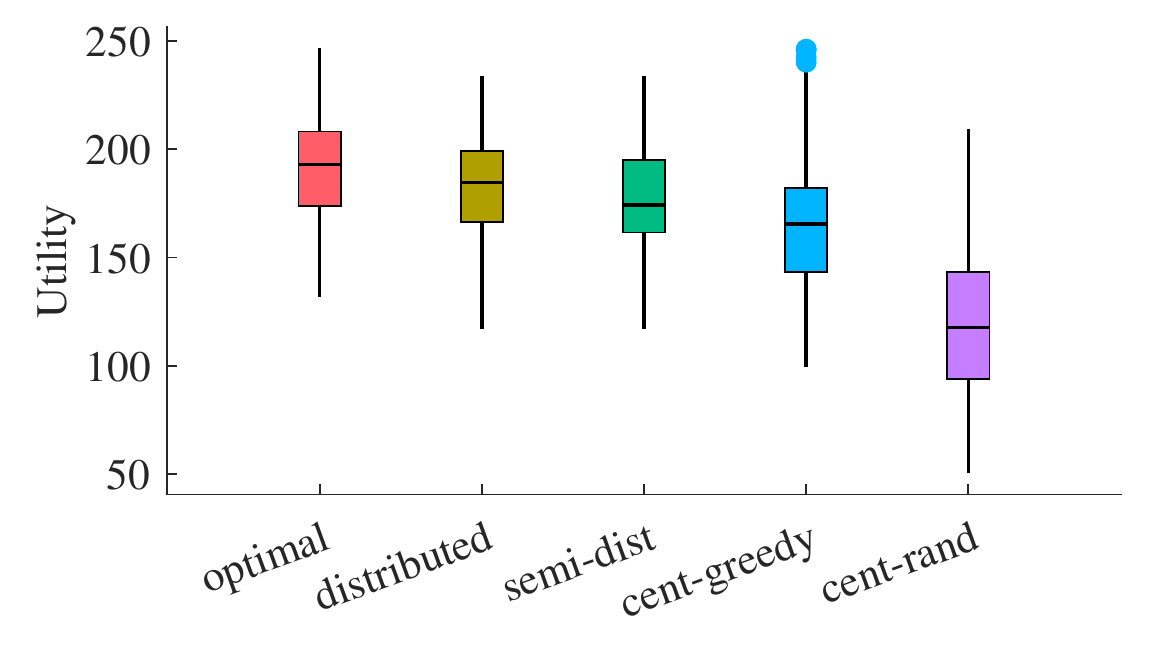}
    \caption{The statistics of the utilities of the five different methods over $200$ trials with the number of robots $N=5$ and the number of attacks $K=3$. The box-plot demonstrates the quartiles of different solutions.}
    \label{fig: box-plot}
\end{figure}

\begin{figure}[!tbp]
    \centering
    \includegraphics[width=3in]{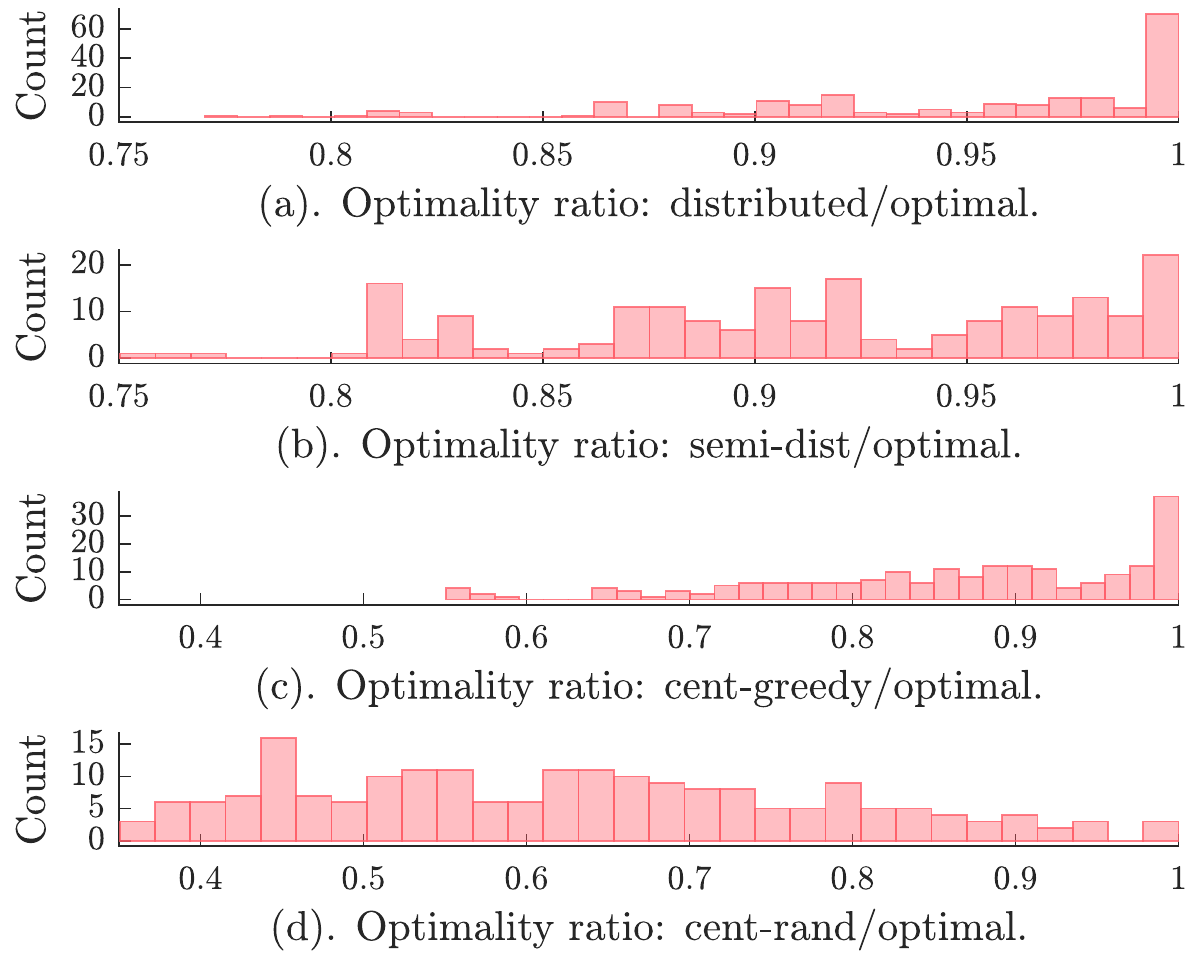}
    \caption{The optimality ratios of different solutions with respect to their corresponding optimal solutions in each of the $200$ trials.}
    \label{fig: hists}
\end{figure}
    In the first setting, we compare the statistics of the utilities of different methods using $200$ trials, as shown in \figref{fig: box-plot}.
    The utilities in the box-plot reflect the performance of different methods by using the quartiles of each solution. As suggested in the result, we observe that the median of the utilities generated by the proposed distributed resilient method shows better performance than that of the other three methods except for the optimal method.
\begin{figure}[!tbp]
\centering
\subfigure[]{
\label{fig: env1}
\includegraphics[width = 1.58in]{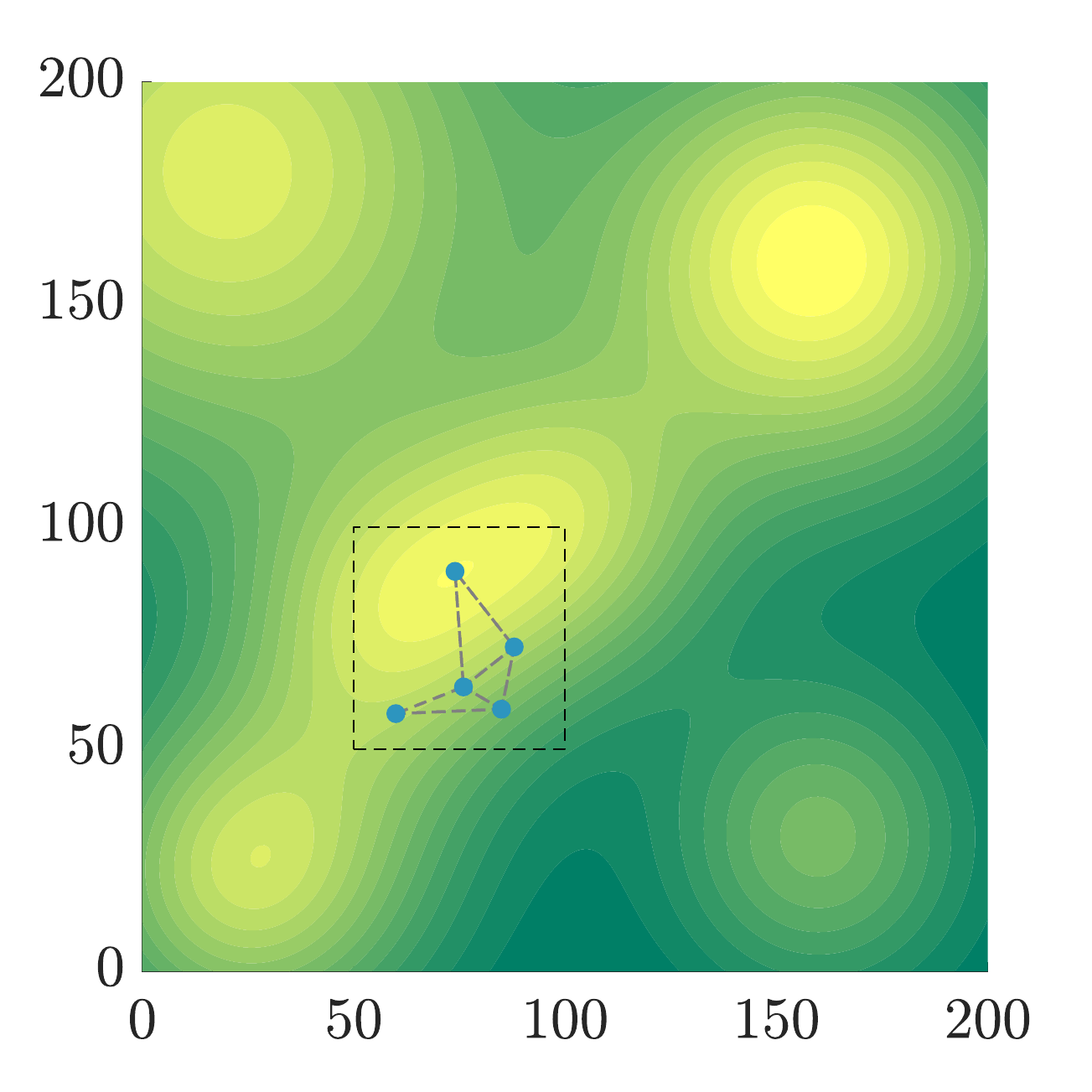}}
\subfigure[]{
\label{fig: env2}
\includegraphics[width = 1.58in]{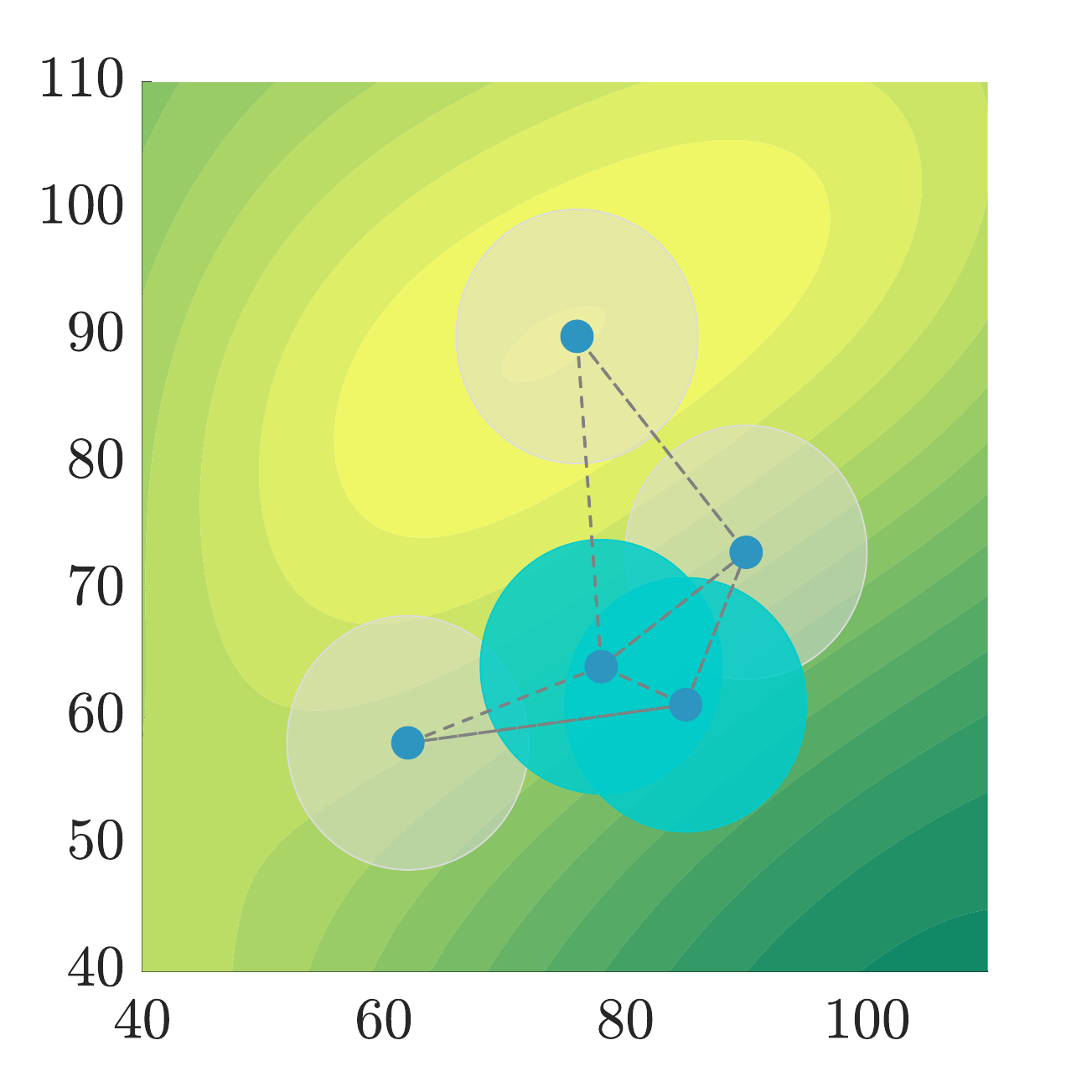}}
\caption{(a). The robots ($N = 5$) are placed randomly in the environment, and a connected communication graph $\G$ is initialized randomly. (b). The resilient solution after the worst-case attacks ($K =3$). The selections of worst-case attacked sensors are in gray. The selections of unattacked sensors are in cyan.}
\label{fig: env}
\end{figure}

In \figref{fig: hists}, we compare the optimality ratios of different solutions with respect to their corresponding optimal solutions in each setting. Specifically, the optimality ratio range of the proposed distributed-resilient is $[0.77, 1]$. The optimality ratio range of the semi-distri-resilient method is $[0.75, 1]$. The optimality ratio range of the centralized-greedy method is $[0.55, 1]$ The optimality ratio range of the centralized-random method is $[0.35, 1]$. This further illustrates that the proposed algorithm (\algref{alg: 1}) is superior to the other methods since most of the cases have a close to optimal optimality ratio as shown in \figref{fig: hists}(a). In \figref{fig: env1}, we demonstrate the environment and the initial configuration of the robots of one instance. Then, the solution of the proposed distributed-resilient method using one instance is shown in \figref{fig: env2}.

\begin{figure}[!tbp]
    \centering
    \includegraphics[width=3.2in]{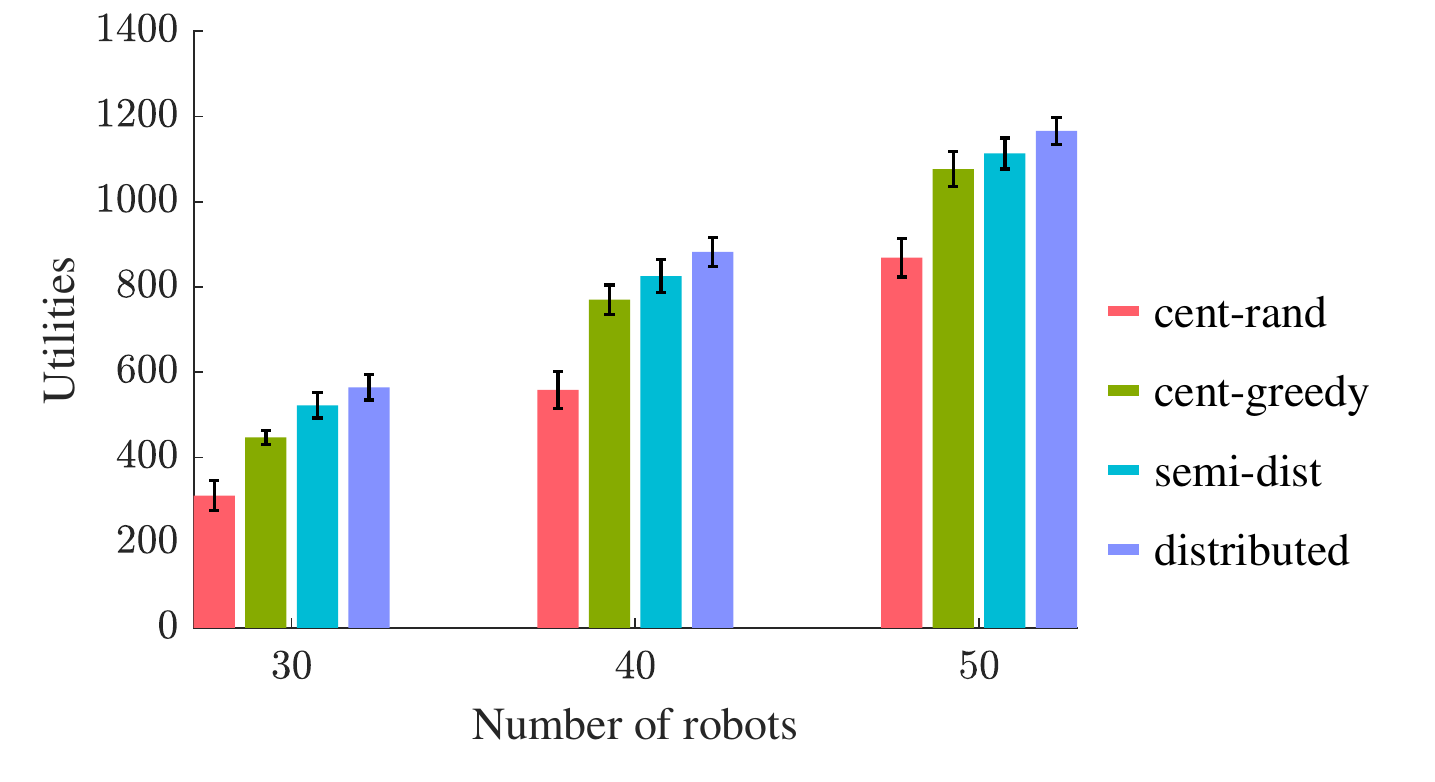}
    \caption{The utilities of the four different methods with $N =30, 40$, and $50$ and with the number of attacks $K$ (for each $N$) randomly generated from $[0.5N, 0.75N]$.}
    \label{fig: 4}
\end{figure}
 
\begin{figure}[!tbp]
    \centering
    \includegraphics[width=3.3in]{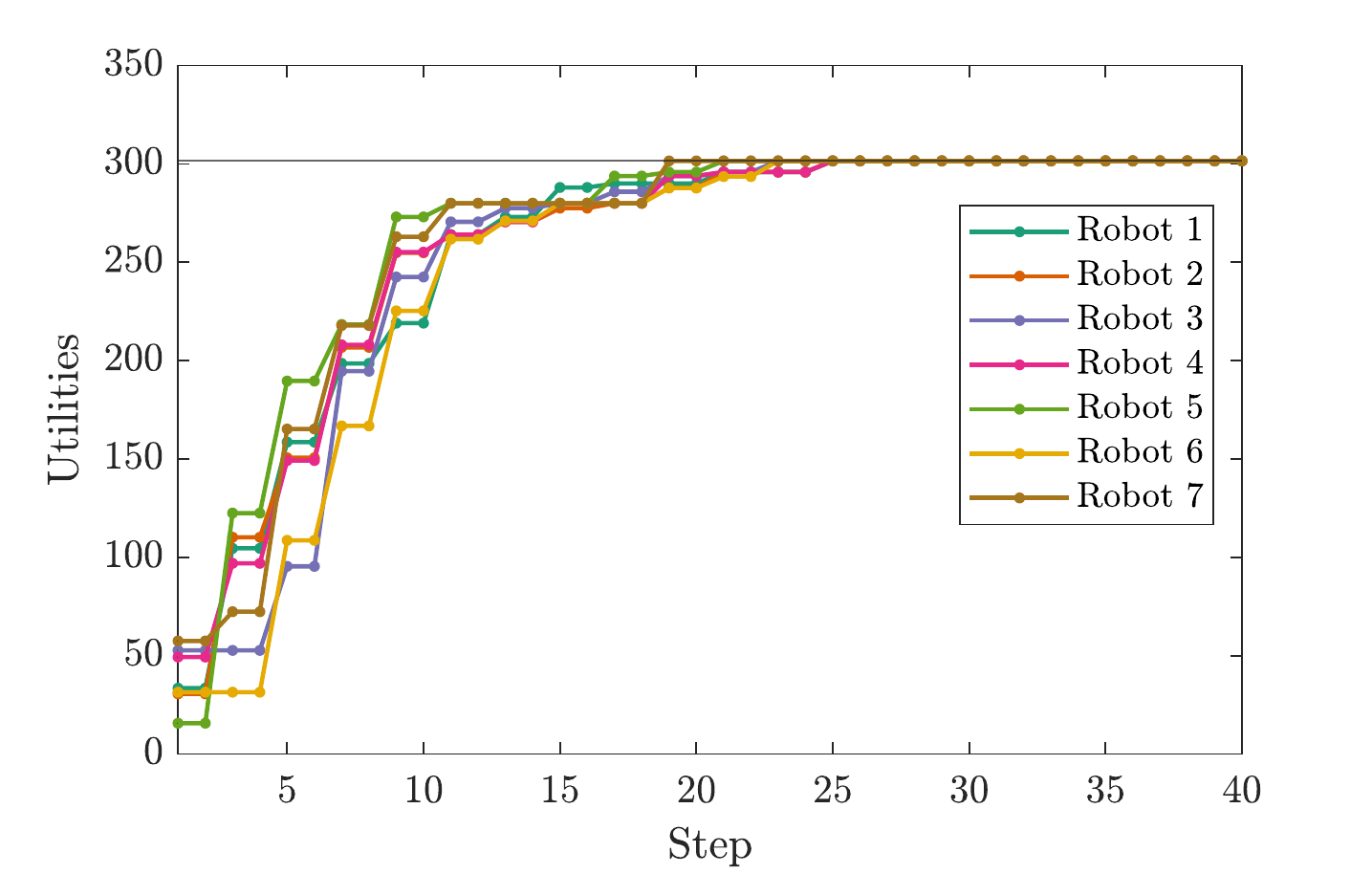}
    \caption{The evolution of the utilities of the 7 robots that are not attacked when the number of robot $N = 15$ robots and the number of attacks $K = 8$.}
    \label{fig: 5}
\end{figure}

In the second setting, we compare the mean of the utilities of different methods.
\figref{fig: 4} shows the utilities of four different approaches. The result demonstrates that \algref{alg: 1} yields superior results compared with the other methods.
Finally, we plot in \figref{fig: 5} the evolution of utilities for different robots when the number of robots is $15$, and the number of attacks is $8$, demonstrating the convergence of the proposed distributed resilient method.

\begin{remark}
     Under our problem formulation, imperfect motion and sensing impact our algorithm in two ways: they disturb how the robots evaluate the reward of their actions and the ability to execute high-reward actions as planned. These influences do not change the fundamental nature of the problem but would instead impact the collected reward. However, the proposed algorithm is still superior to the other three methods, as shown in \figref{fig: 4}. 
\end{remark}

\section{Conclusions and Future Work}
\label{sec: conclusion}

In this letter, we proposed a fully distributed algorithm for the problem of the resilient submodular action selection. We proved that the solution of the proposed algorithm converges to the corresponding centralized algorithm. We evaluated the algorithm's performance through extensive simulations. Directions for future work include exploiting connectivity of the system when communications are attacked, investigating the different importance of the robots (nodes) when the system is attacked, and revisiting our problem with noisy motion and perception considered.

\bibliographystyle{IEEEtran}
\bibliography{ref}

\begin{thebibliography}{10}
\providecommand{\url}[1]{#1}
\csname url@rmstyle\endcsname
\providecommand{\newblock}{\relax}
\providecommand{\bibinfo}[2]{#2}
\providecommand\BIBentrySTDinterwordspacing{\spaceskip=0pt\relax}
\providecommand\BIBentryALTinterwordstretchfactor{4}
\providecommand\BIBentryALTinterwordspacing{\spaceskip=\fontdimen2\font plus
\BIBentryALTinterwordstretchfactor\fontdimen3\font minus
  \fontdimen4\font\relax}
\providecommand\BIBforeignlanguage[2]{{%
\expandafter\ifx\csname l@#1\endcsname\relax
\typeout{** WARNING: IEEEtran.bst: No hyphenation pattern has been}%
\typeout{** loaded for the language `#1'. Using the pattern for}%
\typeout{** the default language instead.}%
\else
\language=\csname l@#1\endcsname
\fi
#2}}
\renewcommand\BIBentryALTinterwordstretchfactor{4}

\bibitem{zhou2021multi}
L.~Zhou and P.~Tokekar, ``Multi-robot coordination and planning in uncertain
  and adversarial environments,'' \emph{Current Robotics Reports}, pp. 1--11,
  2021.

\bibitem{saulnier2017resilient}
K.~Saulnier, D.~Saldana, A.~Prorok, G.~J. Pappas, and V.~Kumar, ``Resilient
  flocking for mobile robot teams,'' \emph{{IEEE} Robot. Autom. Letter},
  vol.~2, no.~2, pp. 1039--1046, 2017.

\bibitem{gil2017guaranteeing}
S.~Gil, S.~Kumar, M.~Mazumder, D.~Katabi, and D.~Rus, ``Guaranteeing
  spoof-resilient multi-robot networks,'' \emph{Auton. Robots}, vol.~41, no.~6,
  pp. 1383--1400, 2017.

\bibitem{mitra2019resilient}
A.~Mitra, J.~A. Richards, S.~Bagchi, and S.~Sundaram, ``Resilient distributed
  state estimation with mobile agents: {O}vercoming {B}yzantine adversaries,
  communication losses, and intermittent measurements,'' \emph{Auton. Robots},
  vol.~43, no.~3, pp. 743--768, 2019.

\bibitem{wardega2019masquerade}
K.~Wardega, R.~Tron, and W.~Li, ``Masquerade attack detection through
  observation planning for multi-robot systems,'' in \emph{Proc. Int. Conf.
  Auton. Agents Multi. Syst.}, 2019, pp. 2262--2264.

\bibitem{raymond2008denial}
D.~R. Raymond and S.~F. Midkiff, ``Denial-of-service in wireless sensor
  networks: Attacks and defenses,'' \emph{IEEE Pervasive Comput.}, vol.~7,
  no.~1, pp. 74--81, 2008.

\bibitem{tzoumas2017resilient}
V.~Tzoumas, K.~Gatsis, A.~Jadbabaie, and G.~J. Pappas, ``Resilient monotone
  submodular function maximization,'' in \emph{Proc. {IEEE} Conf. Decis.
  Control}, 2017, pp. 1362--1367.

\bibitem{zhou2018resilient}
L.~Zhou, V.~Tzoumas, G.~J. Pappas, and P.~Tokekar, ``Resilient active target
  tracking with multiple robots,'' \emph{{IEEE} Robot. Autom. Letter}, vol.~4,
  no.~1, pp. 129--136, 2018.

\bibitem{shi2020robust}
G.~Shi, L.~Zhou, and P.~Tokekar, ``Robust multiple-path orienteering problem:
  Securing against adversarial attacks,'' in \emph{Proc. Robot.: Sci. Syst.},
  2020.

\bibitem{zhou2020distributed}
L.~Zhou, V.~Tzoumas, G.~J. Pappas, and P.~Tokekar, ``Distributed attack-robust
  submodular maximization for multi-robot planning,'' in \emph{Proc. {IEEE}
  Int. Conf. Robot. Autom.}, 2020, pp. 2479--2485.

\bibitem{choi2009consensus}
H.-L. Choi, L.~Brunet, and J.~P. How, ``Consensus-based decentralized auctions
  for robust task allocation,'' \emph{{IEEE} Trans. Robot.}, vol.~25, no.~4,
  pp. 912--926, 2009.

\bibitem{qu2019distributed}
G.~Qu, D.~Brown, and N.~Li, ``Distributed greedy algorithm for multi-agent task
  assignment problem with submodular utility functions,'' \emph{Automatica},
  vol. 105, pp. 206--215, 2019.

\bibitem{williams2017decentralized}
R.~K. Williams, A.~Gasparri, and G.~Ulivi, ``Decentralized matroid optimization
  for topology constraints in multi-robot allocation problems,'' in \emph{Proc.
  IEEE Int. Conf. Robot. Autom.}, 2017, pp. 293--300.

\bibitem{liu2020monitoring}
J.~Liu and R.~K. Williams, ``Monitoring over the long term: Intermittent
  deployment and sensing strategies for multi-robot teams,'' in \emph{Proc.
  IEEE Int. Conf. Robot. Autom.}, 2020, pp. 7733--7739.

\bibitem{liu2020coupled}
J.~Liu and R.~K. Williams, ``Coupled temporal and spatial environment
  monitoring for multi-agent teams in precision farming,'' in \emph{IEEE Conf.
  Control Technol. Appl.}, 2020, pp. 273--278.

\bibitem{corah2019distributed}
M.~Corah and N.~Michael, ``Distributed matroid-constrained submodular
  maximization for multi-robot exploration: Theory and practice,'' \emph{Auton.
  Robots}, vol.~43, no.~2, pp. 485--501, 2019.

\bibitem{liu2019submodular}
J.~Liu and R.~K. Williams, ``Submodular optimization for coupled task
  allocation and intermittent deployment problems,'' \emph{IEEE Robot. Autom.
  Letter}, vol.~4, no.~4, pp. 3169--3176, 2019.

\bibitem{grimsman2018impact}
D.~Grimsman, M.~S. Ali, J.~P. Hespanha, and J.~R. Marden, ``The impact of
  information in distributed submodular maximization,'' \emph{IEEE Trans.
  Control Netw. Syst.}, vol.~6, no.~4, pp. 1334--1343, 2018.

\bibitem{mackin2018submodular}
E.~Mackin and S.~Patterson, ``Submodular optimization for consensus networks
  with noise-corrupted leaders,'' \emph{IEEE Trans. Autom. Control}, vol.~64,
  no.~7, pp. 3054--3059, 2018.

\bibitem{clark2013supermodular}
A.~Clark, L.~Bushnell, and R.~Poovendran, ``A supermodular optimization
  framework for leader selection under link noise in linear multi-agent
  systems,'' \emph{IEEE Trans. Autom. Control}, vol.~59, no.~2, pp. 283--296,
  2013.

\bibitem{nemhauser1978analysis}
G.~L. Nemhauser, L.~A. Wolsey, and M.~L. Fisher, ``An analysis of
  approximations for maximizing submodular set functions—{I},'' \emph{Math.
  Program.}, vol.~14, no.~1, pp. 265--294, 1978.

\bibitem{krause2014submodular}
A.~Krause and D.~Golovin, ``Submodular function maximization,'' in
  \emph{Tractability: Practical Approaches to Hard Problems}. Cambridge
  University Press, 2014, pp. 71--104.

\bibitem{schrijver2003combinatorial}
A.~Schrijver, \emph{Combinatorial optimization: polyhedra and efficiency}.
  Springer Science \& Business Media, 2003, vol.~24.

\bibitem{conforti1984submodular}
M.~Conforti and G.~Cornu{\'e}jols, ``Submodular set functions, matroids and the
  greedy algorithm: tight worst-case bounds and some generalizations of the
  rado-edmonds theorem,'' \emph{Discrete Appl. Math.}, vol.~7, no.~3, pp.
  251--274, 1984.

\end{thebibliography}

\end{document}